\documentclass[sigconf]{acmart}

\usepackage{algorithm}
\usepackage{algorithmic}
\usepackage{graphicx}
\usepackage{textcomp}
\usepackage{xcolor}
\usepackage{url}
\usepackage{xspace}
\usepackage{booktabs}
\usepackage{hyperref}
\usepackage{booktabs}
\usepackage{enumitem}

\usepackage{colortbl}

\usepackage{multirow}
\usepackage[normalem]{ulem}
\useunder{\uline}{\ul}{}
\usepackage{color}
\usepackage{color, xspace}

\newtheorem{theorem}{Theorem}
\newtheorem{lemma}{Lemma}

\AtBeginDocument{%
  \providecommand\BibTeX{{%
    \normalfont B\kern-0.5em{\scshape i\kern-0.25em b}\kern-0.8em\TeX}}}

\copyrightyear{2026}
\acmYear{2026}
\setcopyright{cc}
\setcctype{by}
\acmConference[WWW Companion '26]{Companion Proceedings of the ACM Web Conference 2026}{April 13--17, 2026}{Dubai, United Arab Emirates}
\acmBooktitle{Companion Proceedings of the ACM Web Conference 2026 (WWW Companion '26), April 13--17, 2026, Dubai, United Arab Emirates}
\acmPrice{}
\acmDOI{10.1145/3774905.3794660}
\acmISBN{979-8-4007-2308-7/2026/04}

\begin{document}

\title{Post-Training Fairness Control: A Single-Train Framework for Dynamic Fairness in Recommendation}


\author{Weixin Chen}
\affiliation{
  \institution{Hong Kong Baptist University}
  \city{Hong Kong}
  \country{China}
}
\email{cswxchen@comp.hkbu.edu.hk}

\author{Li Chen}
\affiliation{
  \institution{Hong Kong Baptist University}
  \city{Hong Kong}
  \country{China}
}
\email{lichen@comp.hkbu.edu.hk}

\author{Yuhan Zhao}
\affiliation{
  \institution{Hong Kong Baptist University}
  \city{Hong Kong}
  \country{China}
}
\email{csyhzhao@comp.hkbu.edu.hk}

\begin{abstract}

Despite growing efforts to mitigate unfairness in recommender systems, existing fairness-aware methods typically fix the fairness requirement at training time and provide limited post-training flexibility. 
However, in real-world scenarios, diverse stakeholders may demand differing fairness requirements over time, so retraining for different fairness requirements becomes prohibitive.  
To address this limitation, we propose Cofair, a single-train framework that enables post-training fairness control in recommendation.
Specifically, Cofair introduces a shared representation layer with fairness-conditioned adapter modules to produce user embeddings specialized for varied fairness levels, along with a user-level regularization term that guarantees user-wise monotonic fairness improvements across these levels. 
We theoretically establish that the adversarial objective of Cofair upper bounds demographic parity and the regularization term enforces progressive fairness at user level.
Comprehensive experiments on multiple datasets and backbone models demonstrate that our framework provides dynamic fairness at different levels, delivering comparable or better fairness-accuracy curves than state-of-the-art baselines, without the need to retrain for each new fairness requirement.
Our code is publicly available at \url{https://github.com/weixinchen98/Cofair/}.
\end{abstract}


\maketitle

\section{Introduction}
\label{sec:introduction}



Recommender systems play a pivotal role in assisting users' decision-making process across a broad range of domains, including e-commerce, social media, and entertainment platforms~\cite{linden2003amazon, covington2016deep, wang2018billion, ZCC24, guy2010social}. 
While effective personalization can enhance user satisfaction and platform engagement, it also risks amplifying societal biases inherent in the data, disproportionately affecting particular groups of users based on sensitive attributes such as gender, race, or age~\cite{lifairness, wang2022survey}. 
To this end, fairness-aware approaches have been explored by integrating fairness constraints into the training procedure, typically through learning fair representations of users that are independent of their sensitive attributes~\cite{DBLP:conf/sigir/WuXZZ0ZL022, DBLP:conf/icml/MadrasCPZ18, chen2025causality}. 
However, existing fairness approaches often fix fairness requirements at training time, forcing complete retraining whenever fairness needs evolve, which is a costly bottleneck in real-world deployments.

Recently, some approaches have begun to address the inflexibility issue in fairness deployment post-training. 
For instance, Li et al.~\cite{PCFR} present a personalized fairness approach that enables users to specify the sensitive attributes that would be independent of their recommendations. AFRL~\cite{zhu2024adaptive} adaptively learns fair user representations by treating fairness requirements as inputs with information alignment, allowing users or system developers to select particular attributes to protect. 
However, while these methods offer more flexibility in terms of selecting \emph{which} attributes to protect, they have limited control over \emph{how much} the fairness criteria (a.k.a. different fairness levels) to impose once training has finished, failing to support on-the-fly tuning of fairness degrees at inference time. 
On the other hand, though some studies~\cite{song2019learning, cui2023controllable} provide theoretical fairness guarantees by allowing stakeholders to specify unfairness limits in training. Though their constrained optimization frameworks guarantee controllable fairness during training, they still lack post-training controllability as stakeholders must retrain the entire model to obtain outputs for different fairness levels, which is computationally prohibitive in real-world scenarios.

To this end, we propose \emph{Cofair}, a novel single-training approach that supports post-training fairness adjustments, thereby eliminating the need for repeated end-to-end retraining.
In particular, it consists of a shared representation layer together with fairness-conditioned adapters to progressively optimize different fairness levels while enforcing monotonic improvements in fairness for each user. 
The shared representation layer is designed to capture user characteristics and common patterns essential for balancing accuracy and fairness across various settings, while the fairness-conditioned adapters aim to adjust user representations for specific fairness levels. 
To enforce progressive fairness constraints, we introduce a user-level regularization term to ensure that no individual user’s fairness degrades at higher fairness levels. 
In addition to its intuitive effectiveness in providing progressively fair performance across multiple levels, we formally guarantee this capability from a theoretical perspective by establishing two key statements.
First, we show that the adversarial fairness objective in Cofair upper bounds the group-fairness criterion of demographic parity in lemma, thereby suggesting that minimizing the adversarial fairness loss correlates with a reduction in group disparity. Moreover, our approach naturally accommodates extension to other fairness notions (e.g., equal opportunity) by tailoring the adversarial objective accordingly. 
Second, assisted by the prior lemma, Cofair further enforces non-decreasing fairness improvement for each user as the fairness level increases, guaranteed by the convergence of the user-level regularization.


In summary, the key contributions of this work are four-fold:
\begin{itemize}
    \item We first propose a controllable fairness framework, \emph{Cofair}, that supports adjustable fairness levels via adapter modules, offering post-training flexibility in real-world deployment.
    \item We propose a shared representation layer to capture common patterns for fairness-accuracy balances across levels and a set of fairness-conditioned adapters tailored to specific fairness levels, coupled with user-level regularization that enforces progressive fairness improvements.
    \item We provide theoretical analysis establishing that our adversarial objective upper bounds group fairness criterion (e.g., demographic parity) and Cofair enforces monotonic fairness guarantees for each user.
    \item We conduct extensive experiments against state-of-the-art fairness baselines and demonstrate that \emph{Cofair} delivers controllable fairness at multiple levels with comparable or better fairness-accuracy curves, without retraining.
\end{itemize}


\section{Preliminaries}
\label{sec:preliminaries}
In this section, we introduce the formal setting of recommendation tasks, the notion of sensitive attributes, and the fairness definitions. 

\subsection{Recommendation Task}
Let $\mathcal{U} = \{1, 2, \dots, U\}$ be the set of users and $\mathcal{I} = \{1, 2, \dots, I\}$ be the set of items. Each user $u \in \mathcal{U}$ interacts with a subset of items. We denote the observed user-item interaction set with negative sample by $\mathcal{D} = \{(u,i,j)\}$, where $(u,i,j) \in \mathcal{D}$ indicates a positive interaction $(u,i)$ such as a purchase or click, with a negative sample $(u,j)$ for pair-wise optimization. 
Throughout, we assume that our model is based on user embeddings $\{\mathbf{e}_u\}$ derived from a base recommendation backbone such as BPR~\cite{BPR} or LightGCN~\cite{he2020lightgcn}. 
The goal of a recommender system is to learn a scoring function $\hat{y}_{ui}$ that ranks items for each user $u$ in descending order of predicted preference.
We denote by $\mathcal{L}_{\text{rec}}$ the standard recommendation loss (e.g., BPR loss), which aims to maximize ranking performance.

\subsection{Fairness Task}
We assume each user $u$ is associated with a binary sensitive attribute $a_u \in \{0, 1\}$, such as genders \emph{male} vs. \emph{female}. For group-fairness metrics, we often separate users into two subgroups:
\begin{equation}
G_0 = \{u \mid a_u = 0\}, \quad
G_1 = \{u \mid a_u = 1\}.
\end{equation}
Our approach can naturally extend to non-binary or intersectional attributes by employing multi-dimensional adversarial networks.

A central fairness notion in this paper is \emph{demographic parity} (DP)~\cite{dwork2012fairness, zemel2013learning}, aiming to ensure that the model's predictions are independent of sensitive attributes. Given a recommendation function $G(\mathbf{e}_u)$, the demographic parity difference $\Delta_{\text{DP}}$ is typically defined as:
\begin{equation}
\Delta_{\text{DP}} = \left| \mathbb{E}_{u \in G_1} [G(\mathbf{e}_u)] - \mathbb{E}_{u \in G_0} [G(\mathbf{e}_u)] \right|.
\end{equation}
Smaller $\Delta_{\text{DP}}$ indicates reduced disparity across subgroups. In a top-$K$ recommendation context, DP can also be measured by
comparing the recommendation lists distributed equally across groups~\cite{fairmi}.
Although our focus is on DP for concreteness, the proposed framework is flexible enough to accommodate other fairness notions, such as Equal Opportunity~\cite{DBLP:conf/nips/HardtPNS16}, as discussed in Section~\ref{sec:other_fairness_criteria}.

To integrate fairness, we introduce a fairness loss $\mathcal{L}_{\text{fair}}$ typically enforced via an adversarial network. In Section~\ref{sec:method}, we explain how we partition these components across multiple fairness levels and apply user-level regularization to maintain progressive constraints.



\section{Methodology}
\label{sec:method}

In this section, we present our \emph{controllable fairness} framework, which allows the recommender system to produce a range of fairness-accuracy trade-offs after a single training cycle. The key idea is to introduce a \emph{shared} user representation combined with \emph{fairness-conditioned adapter modules}, each targeting a different fairness level.
We then describe our user-level regularization, ensuring progressive fairness for each user across levels.

\subsection{Model Architecture}\label{sec:method-architecture}
Let $\mathbf{e}_u \in \mathbb{R}^d$ be the original embedding of user $u$, derived from a standard baseline model (e.g., BPR or LightGCN). We aim to transform $\mathbf{e}_u$ into \emph{multiple} embeddings $\mathbf{e}_u^{(t)}$, each calibrated to a different fairness level $t \in \{1,\dots,T\}$. To do so, we define three components:

\subsubsection{Shared Representation Layer}
To efficiently capture user characteristics that are common across different fairness requirements while avoiding redundant learning, we introduce a shared representation layer that serves as the foundation for all fairness levels. This design follows the principle of multi-task learning~\cite{caruana1997multitask, wang2021understanding}, where shared knowledge can benefit multiple related tasks—in our case, recommendation under different fairness constraints.

We map the original embedding $\mathbf{e}_u$ into a lower-dimensional \emph{shared} embedding $\mathbf{s}_u \in \mathbb{R}^{d_s}$:
\begin{equation}
    \mathbf{s}_u = S(\mathbf{e}_u; \theta_{s}),
\end{equation}
where $S$ is a neural network with parameters $\theta_s$. 
This shared architecture
acts as an efficient dimension reduction mechanism~\cite{kusupati2022matryoshka} that distills essential user characteristics common across different fairness levels. 
By maintaining a single shared layer, it provides a stable foundation for learning fairness-invariant features, making the model more robust to fairness adjustments. 
What's more, this design also reduces the model's memory footprint while maintaining its expressive power, as common patterns need only be learned once rather than separately for each fairness level.

\subsubsection{Fairness-Conditioned Adapters} 
Different stakeholders may require varying degrees of fairness in recommendations, from minimal intervention to strict equality across groups. 
To accommodate this spectrum of requirements without compromising the shared user characteristics, we employ specialized adapter modules for each fairness level.
Specifically, for each fairness level $t$, we define an \emph{adapter} network $P^{(t)}$, producing an adapter embedding $\mathbf{p}_u^{(t)} \in \mathbb{R}^{d_p}$:
\begin{equation}
    \mathbf{p}_u^{(t)} = P^{(t)}\bigl(\mathbf{e}_u; \theta_{p}^{(t)}\bigr).
\end{equation}
By employing parallel adapters, the framework can dynamically switch to any desired fairness regime after aligning with the shared representation. This modular architecture not only preserves the core user characteristic (learned in the shared layer) but also limits fairness-specific alterations to each adapter. Consequently, each adapter learns its own ``knob'' to dial in the level of debiasing, thus offering a flexible pathway for controlling fairness across a broad range of user protection requirements.

\subsubsection{Output Layer} The final user embedding $\mathbf{e}_u^{(t)}$ is obtained by concatenating the shared embedding $\mathbf{s}_u$ and the adapter embedding $\mathbf{p}_u^{(t)}$, followed by an output transformation $O(\cdot)$:
\begin{equation}
    \mathbf{e}_u^{(t)} = O\Bigl(\bigl[\mathbf{s}_u ; \mathbf{p}_u^{(t)}\bigr]; \theta_o\Bigr).
\end{equation}
Here, $\theta_o$ are the parameters of the output layer, and $[\cdot;\cdot]$ indicates concatenation. As discussed, we keep the shared portion $\mathbf{s}_u$ to minimize redundancy, while $P^{(t)}$ injects fairness-specific adjustments.

Notably, Cofair is a model-agnostic framework, compatible with various recommendation backbones such as BPR and LightGCN.

\subsection{Loss Functions}\label{sec:method-losses}
Our framework comprises three key losses: (1) a standard \emph{recommendation loss}, (2) an \emph{adversarial fairness loss}, and (3) a \emph{user-level regularization} term.

\subsubsection{Recommendation Loss}
Denote by $\hat{y}_{ui}^{(t)}$ the predicted score for user $u$ and item $i$ at fairness level $t$, derived from $\mathbf{e}_u^{(t)}$. We employ the Bayesian Personalized Ranking (BPR) loss~\cite{BPR}:
\begin{equation}
\label{eq:rec_loss}
\mathcal{L}_{\text{rec}}^{(t)} = - \sum_{(u,i,j) \in \mathcal{D}} \ln \sigma\bigl(\hat{y}_{u i}^{(t)} - \hat{y}_{u j}^{(t)}\bigr),
\end{equation}
where $\sigma(\cdot)$ is the sigmoid function. Our method can be similarly integrated with other recommendation backbones (e.g., LightGCN~\cite{he2020lightgcn}).

\subsubsection{Fairness Loss via Adversarial Network}
To mitigate differences in embeddings for users with different sensitive attributes, we adopt an adversarial network $D(\cdot;\theta_d)$~\cite{goodfellow2014generative, bose2019compositional} that predicts $a_u$ from $\mathbf{e}_u^{(t)}$. Let $a_u \in \{0,1\}$ be the binary attribute for user $u$. We define:
\begin{equation}
\label{eq:fair_loss}
\mathcal{L}_{\text{fair}}^{(t)} = - \sum_{u \in \mathcal{U}} \ell_{\text{BCE}}\bigl(D(\mathbf{e}_u^{(t)};\theta_d),\, a_u\bigr).
\end{equation}
Here, $\ell_{\text{BCE}}$ denotes the binary cross-entropy loss. By \emph{maximizing} this term with respect to $\theta_d$ and \emph{minimizing} it with respect to $(\theta_s,\theta_p^{(t)},\theta_o)$, we enforce an embedding space where $D$ cannot readily discern the sensitive attribute, as a proxy for demographic parity. In Section~\ref{sec:bounding_dp}, we show that minimizing this adversarial loss tightly relates to reducing the demographic parity difference $\Delta_{\text{DP}}$.

\subsubsection{User-Level Regularization}
Our goal is to ensure \emph{progressive} fairness, wherein each user’s fairness strictly improves (or remains the same) as the fairness level $t$ increases. 
Relying solely on group-level metrics could overlook instances where certain individuals become worse off. 
Hence, we impose a penalty if any user’s fairness degrades with a stricter requirement.
To formalize this, we define a per-user fairness loss:
\begin{equation}
\mathcal{L}_{\text{fair}}^{(t)}(u) = -\ell_{\text{BCE}}\bigl(D(\mathbf{e}_u^{(t)};\theta_d),\, a_u\bigr).
\end{equation}
We then incorporate a \emph{user-level regularizer} to penalize scenarios where fairness at level $t+1$ is worse than at level $t$:
\begin{equation}
\label{eq:user_reg}
\mathcal{L}_{\text{reg}} = \sum_{u \in \mathcal{U}} \sum_{t=1}^{T-1} \text{softplus}\Bigl(\mathcal{L}_{\text{fair}}^{(t+1)}(u) - \mathcal{L}_{\text{fair}}^{(t)}(u)\Bigr).
\end{equation}
Minimizing this term enforces that each user’s fairness loss \emph{does not increase} when moving to a stricter fairness level.
The user-level regularization serves as a critical component for ensuring fairness consistency at the individual user level, penalizing any degradation in per-user fairness as the fairness level increases.
This guarantees that stricter fairness levels provide monotonically improving fairness for each user, not just on \emph{average}, leading to more equitable and stable recommendations across the user base. 
Additionally, this regularization may help prevent fairness oscillations during training, promoting smoother convergence.

\subsection{Adaptive Weighting of the Fairness Loss}
While each level $t$ has a fairness loss $\mathcal{L}_{\text{fair}}^{(t)}$, we introduce a single scalar $\lambda_t$ that balances fairness vs. recommendation in the overall training objective.
A naive approach would arbitrarily fix $\lambda_t$ or define them manually for each $t$. 
However, simply fixing different fairness coefficients for each level manually (i.e., a naive approach) does not scale well when $T$ (the number of fairness levels) is large. 
Tuning each coefficient by hand would require intensive hyperparameter search for every new requirement, and it is not trivial to guarantee the desired degree of fairness improvement from level to level. 
In contrast, we introduce an adaptive weighting mechanism to dynamically adjust each fairness coefficient \(\lambda_t\) based on the current progress of the model training:
\begin{equation}
\label{eq:lambda_update}
\lambda_{t+1} = \lambda_t \;+\; \eta \;\Bigl[\,1 - \frac{\mathcal{L}_{\text{fair}}^{(t)} - \mathcal{L}_{\text{fair}}^{(t+1)}}{\mathcal{L}_{\text{fair}}^{(t)}}\Bigr],
\end{equation}
where \(\eta\) is a small learning rate for the fairness coefficients. Intuitively, if the fairness loss at level $t+1$ is not sufficiently improved over level $t$, then \(\lambda_{t+1}\) is decreased, making fairness more critical in the subsequent training updates. 
This adaptive weighting mechanism reduces manual effort and ensures a smooth fairness-accuracy trade-off across various levels.

\subsection{Overall Objective}
Combining all components, we minimize with respect to \(\Theta = \{\theta_s, \{\theta_p^{(t)}\}_{t=1}^T, \theta_o\}\) and maximize with respect to \(\theta_d\):
\begin{equation}
\label{eq:overall_obj}
\min_{\Theta} \max_{\theta_d} \Bigl\{ 
\mathcal{L} = \frac{1}T{}\sum_{t=1}^{T} \Bigl[\mathcal{L}_{\text{rec}}^{(t)} + \lambda_t \,\mathcal{L}_{\text{fair}}^{(t)}\Bigr]
\;+\;\beta \,\mathcal{L}_{\text{reg}} \Bigr\},
\end{equation}
where \(\beta\) is a hyperparameter to control the strength of the user-level regularization.

\subsection{Training Procedure and Implementation}
We alternate between updating the adversarial network $D$ to better predict $a_u$, and updating the shared and adapter networks to \emph{fool} $D$. We also periodically update $\lambda_t$ via Eq.~\eqref{eq:lambda_update} and back-propagate through the user-level regularization Eq.~\eqref{eq:user_reg}.

In each training epoch, Cofair processes user embeddings through $T$ parallel adapter modules, incurring $T$ forward passes per epoch. However, this does not substantially inflate overall runtime, as the shared representation layer is computed once and each adapter network is typically lightweight (single-layer MLP in our experiments).
At inference time, we feed $\mathbf{e}_u$ into \emph{one} of the $T$ adapter modules, selecting the desired fairness level. 
Since all $T$ fairness levels are learned simultaneously, Cofair only needs to train once, thereby eliminating the multiple full retraining runs required by other methods. This design leads to a practical and scalable solution, as also demonstrated by the efficiency measurements in Section~\ref{sec:efficiency_analysis}.

\section{Theoretical Analysis}
\label{sec:theorem}


We now establish that our adversarial fairness objective theoretically bounds demographic parity (DP) and that the user-level regularization enforces progressive fairness improvements.
Further, we discuss the extension to other fairness criteria by modifying our fairness objective accordingly.

\subsection{Bounding Demographic Parity}
\label{sec:bounding_dp}

Let \( \mathcal{Z}_0 \) and \( \mathcal{Z}_1 \) denote the distributions of the user representations \( \mathbf{e}_u^{(t)} \) conditioned on \( a_u = 0 \) and \( a_u = 1 \), respectively. For a function \( G \) mapping the representations to predictions \( \hat{y}_u^{(t)} = G(\mathbf{e}_u^{(t)}) \)\footnote{In recommender systems, the outcomes involve items. Here, without loss of generality, we focus on the predicted preference scores \( \hat{y}_u^{(t)} \) in our analysis for simplicity.}, the demographic parity difference is expressed as:

\begin{equation}
\Delta_{\text{DP}}^{(t)} \triangleq \left| \mathbb{E}_{\mathcal{Z}_1}[G(\mathbf{e}_u^{(t)})] - \mathbb{E}_{\mathcal{Z}_0}[G(\mathbf{e}_u^{(t)})] \right|.
\end{equation}

From Equation~\eqref{eq:fair_loss}, by omitting the logarithmic terms, the fairness loss over $\mathcal{Z}_0$ and $\mathcal{Z}_1$ is expressed as:

\begin{equation}
\mathcal{L}_{\text{fair}}^{(t)} = \mathbb{E}_{\mathcal{Z}_0} [ 1 - D(\mathbf{e}_u^{(t)}) ] + \mathbb{E}_{\mathcal{Z}_1} [ D(\mathbf{e}_u^{(t)}) ].
\end{equation}

Then, we have the following propositions.

\begin{lemma}
Consider any measurable function \( G : \mathbf{e}_u^{(t)} \to [0,1] \) representing the predicted preference. Then, at fairness level $t$, the demographic parity difference \( \Delta_{\text{DP}}^{(t)} \) is upper bounded by the optimal value of adversarial fairness objective:

\begin{equation}\label{eq:dp_upper_bound}
\Delta_{\text{DP}}^{(t)} \leq \sup_{D} \mathcal{L}_{\text{fair}}^{(t)}.
\end{equation}
\end{lemma}

\begin{proof}
Without loss of generality (WLOG), suppose $\mathbb{E}_{\mathcal{Z}_0}[G(\mathbf{e}_u^{(t)})] \geq \mathbb{E}_{\mathbf{e}_u^{(t)} \sim \mathcal{Z}_1}[G(\mathbf{e}_u^{(t)})]$, then $\Delta_{\text{DP}}^{(t)} =  \mathbb{E}_{\mathcal{Z}_0}[G(\mathbf{e}_u^{(t)})] - \mathbb{E}_{\mathbf{e}_u^{(t)} \sim \mathcal{Z}_1}[G(\mathbf{e}_u^{(t)})] $.
Consider an adversary \( D(\mathbf{e}_u^{(t)}) = 1 - G(\mathbf{e}_u^{(t)}) \). Then, we have:

\begin{align}
\mathcal{L}_{\text{fair}}^{(t)} &= \mathbb{E}_{\mathcal{Z}_0} [ 1 - (1 - G(\mathbf{e}_u^{(t)}))  ] + \mathbb{E}_{\mathcal{Z}_1} [ 1 - G(\mathbf{e}_u^{(t)})  ] \\
&= \mathbb{E}_{\mathcal{Z}_0} [ G(\mathbf{e}_u^{(t)}) ] + \mathbb{E}_{\mathcal{Z}_1} [  (1 - G(\mathbf{e}_u^{(t)})  ]\\
&= 1 - \mathbb{E}_{\mathcal{Z}_1} [ G(\mathbf{e}_u^{(t)} ) ] + \mathbb{E}_{\mathcal{Z}_0} [ G(\mathbf{e}_u^{(t)} ) ] \\
&= 1 + \left( \mathbb{E}_{\mathcal{Z}_0} [ G(\mathbf{e}_u^{(t)} ) ] - \mathbb{E}_{\mathcal{Z}_1} [ G(\mathbf{e}_u^{(t)} ) ] \right)\\
&= \Delta_{\text{DP}}^{(t)} + 1 \;\;\geq\;\; \Delta_{\text{DP}}^{(t)}.
\end{align}

Since \(\sup_{D} \mathcal{L}_{\text{fair}}^{(t)} \geq \mathcal{L}_{\text{fair}}^{(t)} \), we have:

\begin{equation}
\Delta_{\text{DP}}^{(t)} \leq \sup_{D} \mathcal{L}_{\text{fair}}^{(t)}.
\end{equation}

This completes the proof.
\end{proof}

\noindent This lemma establishes that minimizing the adversarial fairness loss \( \mathcal{L}_{\text{fair}}^{(t)} \) leads to a reduction in demographic parity difference \( \Delta_{\text{DP}}^{(t)} \).

\subsection{Progressive Fairness Enforcement}

We now present the main theorem, asserting that our method achieves a monotonic improvement in fairness performance across fairness levels.

\begin{theorem}

For any user $u \in \mathcal{U}$, let $\mathcal{L}_{\text{fair}}^{(t)}(u) = -\ell_{\text{BCE}}(D(\mathbf{e}_u^{(t)}), a_u)$ be the user-level adversarial loss at fairness level $t$. If the minimizer of Eq.~\eqref{eq:overall_obj} is achieved, then for each \( t \in \{1, 2, \dots, T-1\} \):

\begin{equation}
\Delta_{\text{DP}}^{(t+1)}(u) \leq \Delta_{\text{DP}}^{(t)}(u).
\end{equation}

\end{theorem}

\begin{proof}
From the user-level regularization term \( \mathcal{L}_{\text{reg}} \), we have:

\begin{equation}
\mathcal{L}_{\text{reg}} = \sum_{u \in \mathcal{U}} \sum_{t=1}^{T-1} \text{softplus}\left( \Delta_{\text{DP}}^{(t+1)}(u) - \Delta_{\text{DP}}^{(t)}(u) \right).
\end{equation}

The softplus function \( \text{softplus}(x) = \ln(1 + e^{x}) \) is convex and continuously differentiable, with a derivative \( \sigma(x) = \frac{1}{1 + e^{-x}} \). This regularization penalizes cases where \( \Delta_{\text{DP}}^{(t+1)}(u) > \Delta_{\text{DP}}^{(t)}(u) \).

At convergence, the optimization process seeks to minimize \( \mathcal{L}_{\text{reg}} \), which implies:

\begin{equation}
\Delta_{\text{DP}}^{(t+1)}(u) \leq \Delta_{\text{DP}}^{(t)}(u), \quad \forall u \in \mathcal{U}, \ \forall t \in \{1, 2, \dots, T-1\}.
\end{equation}

Therefore, the fairness performance improves (or remains the same) as the fairness level increases.
\end{proof}

\subsection{Extension to Other Fairness Criteria}\label{sec:other_fairness_criteria}
While our main theoretical focus is on demographic parity, the framework is general. For instance, \emph{Equal Opportunity} (EOpp) demands that true positive rates (TPRs) be equal across groups~\cite{DBLP:conf/nips/HardtPNS16}. We can adapt our adversarial term to incorporate labels of user relevance (or positive feedback) when training $D$ so that it specifically penalizes differences conditional on $y=1$. Formally, one may replace $\ell_{\text{BCE}}$ in Eq.~\eqref{eq:fair_loss} with a specialized loss that predicts $a_u$ \emph{only among users with $y_u=1$}, or incorporate a step that conditions on positive interactions. The user-level regularization in Eq.~\eqref{eq:user_reg} similarly ensures each user’s fairness with respect to EOpp does not degrade at higher $t$. Our theoretical bounding arguments extend under standard assumptions that the new fairness loss upper-bounds the chosen fairness metric~\cite{DBLP:conf/icml/MadrasCPZ18}. Hence, although we focus on DP for clarity, the proposed ``control on the fly’’ strategy naturally accommodates a range of fairness definitions.

While our theoretical analysis assumes the existence of an optimal adversary, in practice, the empirical performance may be influenced by the training dynamics of min-max optimization and the choice of hyperparameters $\beta$ and $\eta$.

\section{Experiments}
\label{sec:experiments}




    
    


We structure our empirical analysis for the following questions:
\begin{itemize}[left=5pt]
    \item \textbf{RQ1:} Can Cofair achieve various fairness-accuracy trade-offs without retraining, and how does it compare to baselines?
    
    \item \textbf{RQ2:} How do the components of our proposed Cofair contribute to the accuracy and fairness performance across levels?
    
    \item \textbf{RQ3:} How do different hyperparameter configurations influence the performance of our proposed Cofair?
    
    \item \textbf{RQ4:} Can Cofair enable post-training controllability of fairness for other fairness methods, demonstrating its generalizability?
    
    \item \textbf{RQ5:} Does Cofair reduce computational overhead?
\end{itemize}

\subsection{Experimental Setup}

\subsubsection{Datasets}

Experiments were conducted on two public datasets:  
\textbf{Movielens-1M~\cite{harper2015movielens}} is a classic movie-rating dataset with dense user-item interactions. We treat any rating above 0 as a positive interaction, following~\cite{fairmi, islam2021debiasing}.
\textbf{Lastfm-360K~\cite{celma2009music}} is a large music recommendation dataset with play records of users from Last.FM. We performed 20-core filtering and sampled a subset following~\cite{fairmi}.

We treat \textit{gender} as sensitive attribute, which is the most widely considered sensitive attribute in the fairness literature~\cite{wang2022survey, lifairness, deldjoo2024fairness}. 
The statistics of the processed datasets are shown in Table~\ref{tab:statistics}.

\begin{table}[htbp]
\centering
\small
\caption{Statistics of the datasets used in experiments.}
\begin{tabular}{@{}l|cccc@{}}
\toprule Dataset&\#Interactions&\#Users&\#Items&Sparsity\\
\midrule
Movielens-1M&$1,000,209$&6,040&3,706&$95.53\%$\\
Lastfm-360K&$2,261,740$&48,386&36,775&$99.87\%$\\
\bottomrule
\end{tabular}
\label{tab:statistics}
\vspace{-2mm}
\end{table}

\subsubsection{Evaluation Protocols}

We evaluate recommendation accuracy with two widely adopted ranking metrics \textbf{Recall@10}~\cite{gunawardana2009survey} and \textbf{NDCG@10}~\cite{jarvelin2017ir}. \textit{Larger} values indicate better recommendation accuracy. We evaluate fairness performance with two widely used group fairness metrics in recommender systems, \textbf{DP@10}~\cite{fairmi} and \textbf{EOpp@10}~\cite{fairmi}, where DP@10 measures if the items in the Top-10 recommendations are distributed equally across user groups and EOpp@10 measures if \textit{relevant} items in the Top-10 recommendations are distributed equally across groups. \textit{Lower} values indicate better fairness performance. We validate our findings by calculating $p$-values to assess statistical significance against the best baseline.





\subsubsection{Base Models}
Following previous research~\cite{ZCL23,ZCH24,ZCH25}, we evaluate fairness methods on two representative recommender backbones: 
\textbf{BPR}~\cite{BPR} that learns user and item embeddings by maximizing pairwise rankings, and \textbf{LightGCN}~\cite{he2020lightgcn} that is a graph-based collaborative filtering framework to refine embeddings via message passing over user-item bipartite graphs.

\subsubsection{Baselines}

We compare our method against the following fairness baselines: \textbf{ComFair~\cite{bose2019compositional}} applies compositional adversarial learning to eliminate sensitive multi-attribute information in user representations via a min-max game. \textbf{FairRec~\cite{FairRec}} decomposes adversarial learning to generate a bias-free user representation with minimized sensitive information and a bias-aware user representation with maximized sensitive information, aiming to obscure sensitive information in recommendation. \textbf{FairGo~\cite{FairGo}} applies compositional filters to both user and item representations, and applies discriminators to explicit user representation and graph-based high-order user representation. \textbf{AFRL~\cite{zhu2024adaptive}} adaptively learns fair representations by treating fairness requirements as inputs, preserves non-sensitive information, and incorporates debiased embeddings to balance fairness and accuracy.

\subsubsection{Implementation Details}
For a fair comparison, all methods are in the same optimization setting: the latent space dimension of 64, the Adam optimizer~\cite{DBLP:journals/corr/KingmaB14} with learning rate of 0.001, batch size of 4096, the 1:1 negative sampling strategy~\cite{wei2021contrastive}, and Xavier initialization~\cite{glorot2010understanding}.
The dimensions $d_s$ and $d_p$ are both 64, matching the latent dimension $d$.
The hyperparameters $\lambda_{0}$, $\eta$, and $\beta$ are tuned within the range of $(0,1]$, based on performance on a validation set.
The shared representation layer $S$, the adapter modules $P^{(t)}$, and the output layer are each implemented as a single-layer feedforward neural network (i.e., 1-layer MLP). 
The adversarial network $D$ is implemented using a 2-layer MLP with LeakyReLU activation~\cite{maas2013rectifier} and a dropout rate of 0.2.
We set $T = 5$ to evaluate fairness-accuracy trade-offs at 5 different fairness levels.
For baselines, we adopt the hyperparameter settings recommended by authors or provided in public implementations or fine-tune them for best performance.

\vspace{-2mm}


\subsection{Overall Performance (RQ1)}

\begin{figure*}[htbp]
    \centering
    \includegraphics[width=0.95\linewidth]{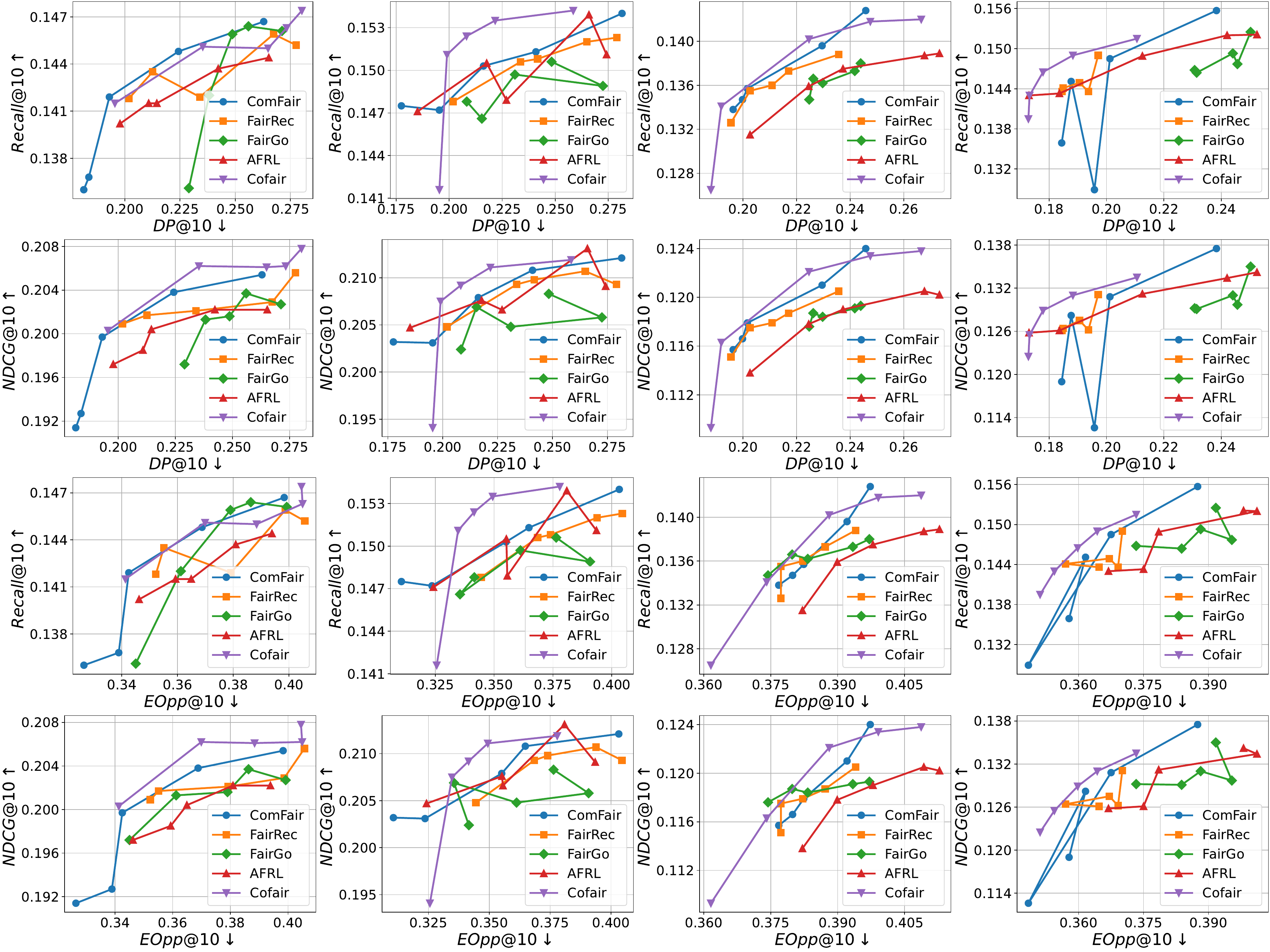}
    
    \begin{minipage}[t]{0.28\linewidth}
        \centering
        \textbf{(a)} Movielens-1M \\ (BPR as backbone)
    \end{minipage}
    \begin{minipage}[t]{0.22\linewidth}
        \centering
        \textbf{(b)} Movielens-1M \\ (LightGCN as backbone)
    \end{minipage}
    \begin{minipage}[t]{0.27\linewidth}
        \centering
        \textbf{(c)} Lastfm-360K \\ (BPR as backbone)
    \end{minipage}
    \hfill
    \begin{minipage}[t]{0.19\linewidth}
        \centering
        \textbf{(d)} Lastfm-360K \\ (LightGCN as backbone)
    \end{minipage}
    
    \vspace{-2mm}
    \caption{Fairness-accuracy curves, where points closer to the left (indicating greater fairness) and the top (indicating higher accuracy), are more Pareto efficient. Notably, the results for baselines are obtained after multiple training, while the results for our Cofair are obtained after single training and multiple forward pass by indicating different values of $t$.}
    \label{fig:overall_performance}
    \vspace{-2mm}
\end{figure*}


Figure~\ref{fig:overall_performance} illustrates the fairness-accuracy performance of our proposed \textbf{Cofair} method compared to four baselines on two benchmark datasets, MovieLens-1M and Lastfm-360K, each examined with two different recommendation backbones (BPR and LightGCN). 
Points in curves that are closer to the left (greater fairness) and higher on the y-axis (stronger accuracy) are generally considered more Pareto-efficient. Note that all baseline curves are obtained by training the models multiple times under different hyperparameter configurations for each fairness level, whereas \textbf{Cofair} generates the entire curve in a single training pass, with multiple inference phases (varying $t$).
Below are three key observations:
\vspace{-1mm}
\begin{itemize}[leftmargin=8pt,labelindent=0pt]
    \item Cofair achieves the most Pareto-efficient trade-offs between fairness and accuracy in 15 out of 16 total comparisons. 
    These improvements over the best baseline(s) are validated to be significant using a two-sided t-test with $p < 0.05$.
    The sole exception is Recall vs. DP with BPR on MovieLens-1M, where Cofair ranks second but remains competitively close to the best-performing method. This overall dominance suggests Cofair effectively balances fairness constraints and recommendation quality across multiple levels without retraining. 

    \item Cofair typically spans a wider range of fairness values (DP@10 and EOpp@10) in a single training run than baselines, which must be retrained multiple times to produce similar curves. Notably, under certain configurations (e.g., LightGCN on Lastfm-360K), some baselines may match or slightly surpass Cofair in the range of attainable fairness. Nevertheless, in most scenarios, Cofair exhibits greater flexibility to navigate trade-offs between extreme fairness and high accuracy. 

    \item While the fairness range of Cofair under the LightGCN backbone on Lastfm-360K is somewhat less extensive than in other settings, it consistently achieves the best DP@10 or EOpp@10 measure at each fairness level with minimal accuracy loss. We conjecture that this is partly due to the inherent adaptability of LightGCN’s embeddings, which enables Cofair to enforce fairness constraints more uniformly across users. Consequently, Cofair attains highly competitive fairness-accuracy trade-offs in this case.
\end{itemize}

In conclusion, these results demonstrate that Cofair offers strong and flexible fairness-accuracy trade-offs across different datasets, backbones, and fairness requirements, without the need for repeated retraining. The framework thus emerges as an effective, efficient, and flexible route to incorporate fairness constraints in recommender systems while maintaining robust performance.

\subsection{Ablation Study (RQ2)}

\begin{table}[t]
\centering
\caption{Ablation study of our proposed method (Cofair) versus its variants without specific components. 
}
\vspace{-2mm}
\label{tab:ablation}
\resizebox{0.95\linewidth}{!}{
\begin{tabular}{lccccccc}
\toprule
\multirow{2}{*}{\textbf{Method}} & \multicolumn{7}{c}{\textbf{NDCG@10 (\textit{larger} is better)}}\\
\cmidrule(lr){2-8}
& $t=1$ & $t=2$ & $t=3$ & $t=4$ & $t=5$ & \textbf{avg.} & \textbf{std.} \\ 
\midrule
\emph{w/o SRL} & 0.2050 & 0.2044 & 0.2051 & 0.2027 & 0.1945 & 0.2023 & \textbf{0.0040} \\ 
\emph{w/o FCA} & 0.2044 & 0.2044 & 0.2044 & \underline{0.2044} & \underline{0.2044} & 0.2044 & 0.0000  \\ 
\emph{w/o AWL} & \underline{0.2073} & \textbf{0.2069} & \textbf{0.2073} & \textbf{0.2076} & \textbf{0.2063} & \textbf{0.2071} & 0.0004  \\ 
\emph{w/o URL} & \textbf{0.2080} & \underline{0.2067} & \underline{0.2064} & 0.2033 & 0.2028 & \underline{0.2054} & 0.0020    \\ 
\textbf{Cofair} & 0.2015 & 0.1996 & 0.1980 & 0.1987 & 0.1913 & 0.1978 & \underline{0.0035}  \\ 
\midrule
\textbf{Improv.} & \cellcolor{gray!15}{-3.13\%} & \cellcolor{gray!15}{-3.43\%} & \cellcolor{gray!15}{-4.07\%} & \cellcolor{gray!15}{-2.26\%} & \cellcolor{gray!15}{-5.67\%} & \cellcolor{gray!15}{-3.70\%} & \cellcolor{gray!15}{+75.00\%} \\
\midrule
\multirow{2}{*}{\textbf{Method}} & \multicolumn{7}{c}{\textbf{DP@10 (\textit{smaller} is better)}}\\
\cmidrule(lr){2-8}
& $t=1$ & $t=2$ & $t=3$ & $t=4$ & $t=5$ & \textbf{avg.} & \textbf{std.} \\
\midrule
\emph{w/o SRL} & 0.2828 & 0.2716 & 0.2623 & \underline{0.2505} & 0.2403 & 0.2615 & 0.0150  \\
\emph{w/o FCA} & \textbf{0.2615} & \underline{0.2615} & \underline{0.2615} & 0.2615 & 0.2615 & 0.2615 & 0.0000  \\
\emph{w/o AWL} & 0.2802 & 0.2785 & 0.2832 & 0.2803 & 0.2783 & 0.2801 & 0.0018 \\
\emph{w/o URL} & 0.2727 & 0.2691 & 0.2652 & 0.2601 & \underline{0.2068} & \underline{0.2548} & \underline{0.0244}   \\
\textbf{Cofair} & \underline{0.2707} & \textbf{0.2508} & \textbf{0.2329} & \textbf{0.2017} & \textbf{0.1891} & \textbf{0.2290} & \textbf{0.0302}    \\
\midrule
\textbf{Improv.} 
& \cellcolor{gray!15}{+0.73\%} 
& \cellcolor{gray!15}{+6.80\%}  
& \cellcolor{gray!15}{+12.18\%} 
& \cellcolor{gray!15}{+22.46\%} 
& \cellcolor{gray!15}{+8.56\%}  
& \cellcolor{gray!15}{+10.13\%}
& \cellcolor{gray!15}{+23.77\%} \\
\bottomrule
\end{tabular}
}
\end{table}

Overall, Table~\ref{tab:ablation} illustrates that \textbf{Cofair} achieves significantly lower demographic parity difference (DP@10) than its ablated variants, reducing DP@10 by 10.13\% on average relative to the best competitor (``w/o URL''). This improvement comes at the cost of a moderate 3.70\% decrease in NDCG@10, which aligns with our design goal of enforcing stronger fairness constraints.

We further analyze the contribution of each component. Removing the shared representation layer (``w/o SRL'') leads to a notable decrease in fairness, indicating its role in capturing common patterns. The fairness-conditioned adapters (FCA) are critical for controllability; without them (``w/o FCA''), the performance collapses to a single regime with identical results across all $t$. Similarly, removing the adaptive weighting loss (``w/o AWL'') results in static fairness coefficients that fail to show clear progression as $t$ increases. Finally, the user-level regularization (``w/o URL'') is essential for user-level consistency; its absence weakens the fairness improvement at higher levels, confirming the effectiveness of enforcing per-user constraints.

\subsection{Hyperparameter Analysis (RQ3)}

We investigate the effects of hyperparameters $\lambda_0$, $\eta$, and $\beta$ on the mean ($\mu$) and variance ($\sigma^2$) of performance across fairness levels, as shown in Figure~\ref{fig:param_sensitivity}. For the initial fairness coefficient $\lambda_0$, increasing it leads to improved average fairness (lower $\mu_{DP}$) but diminished differentiation across levels (lower $\sigma^2_{DP}$). Since larger $\lambda_0$ also results in higher accuracy instability (higher $\sigma^2_{NDCG}$), a smaller $\lambda_0$ is preferred to balance fairness differentiation with stable accuracy.

\begin{figure}[htbp]
    \centering
    \includegraphics[width=0.95\linewidth]{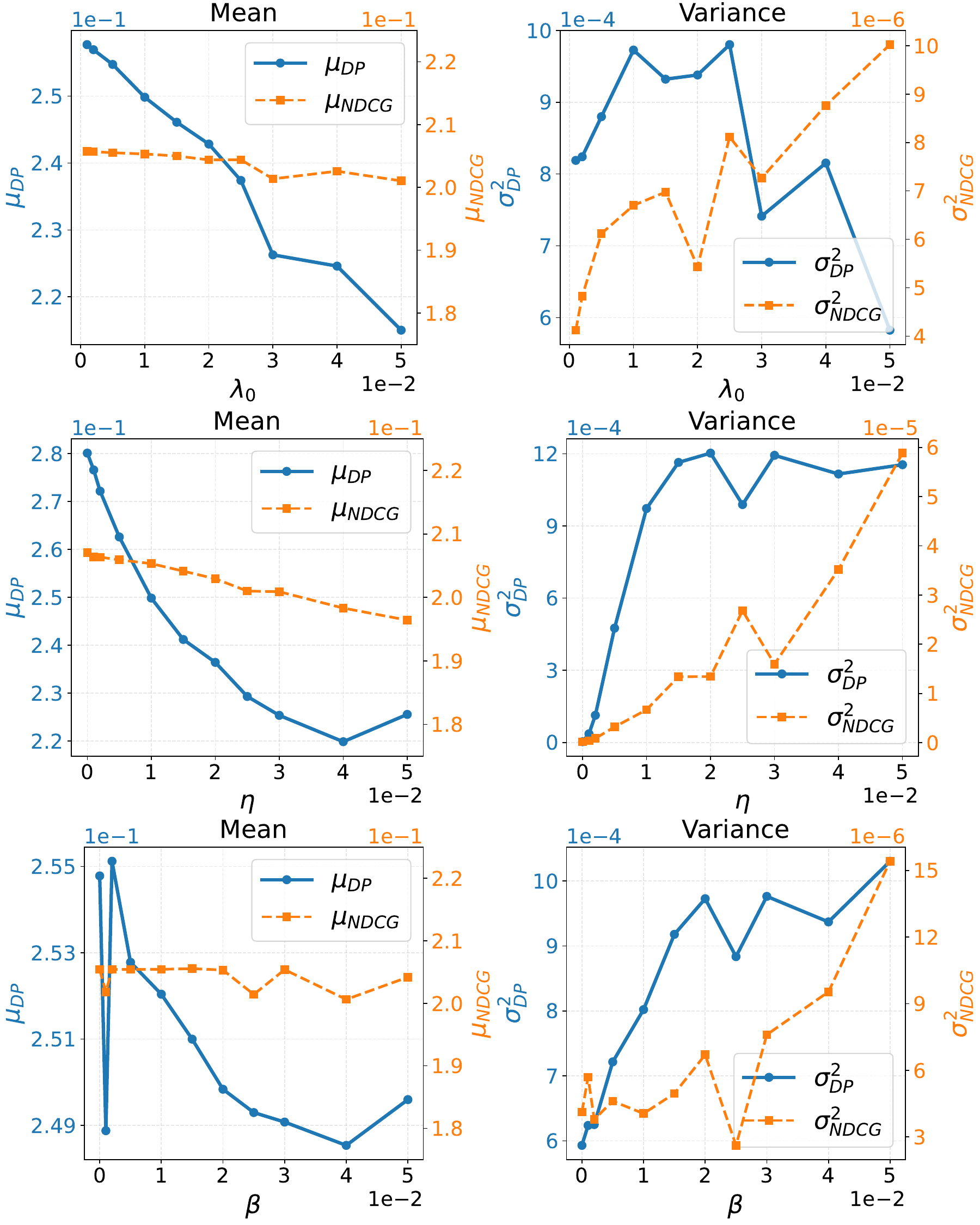}
    \vspace{-2mm}
    \caption{Means and variances of DP ($\mu_{DP}$ and $\sigma^2_{DP}$) and NDCG ($\mu_{NDCG}$ and $\sigma^2_{NDCG}$) on BPR backbone and Movielens-1M dataset with different values of $\lambda_0$, $\eta$, and $\beta$, respectively.}
    \label{fig:param_sensitivity}
    \vspace{-2mm}
\end{figure}

Regarding the update rate $\eta$, larger values facilitate stronger fairness and distinct differentiation (higher $\sigma^2_{DP}$) but at the cost of reduced accuracy, suggesting a moderate $\eta$ yields the optimal trade-off. Similarly, increasing the user-level regularization weight $\beta$ improves fairness parallel to $\eta$. However, its impact on accuracy is less strictly monotonic, implying that a medium $\beta$ can effectively enhance user-level fairness constraints without ensuring a significant drop in average recommendation performance.

\subsection{Framework Study (RQ4)}

\begin{figure}[htbp]
    \centering
    \includegraphics[width=0.95\linewidth]{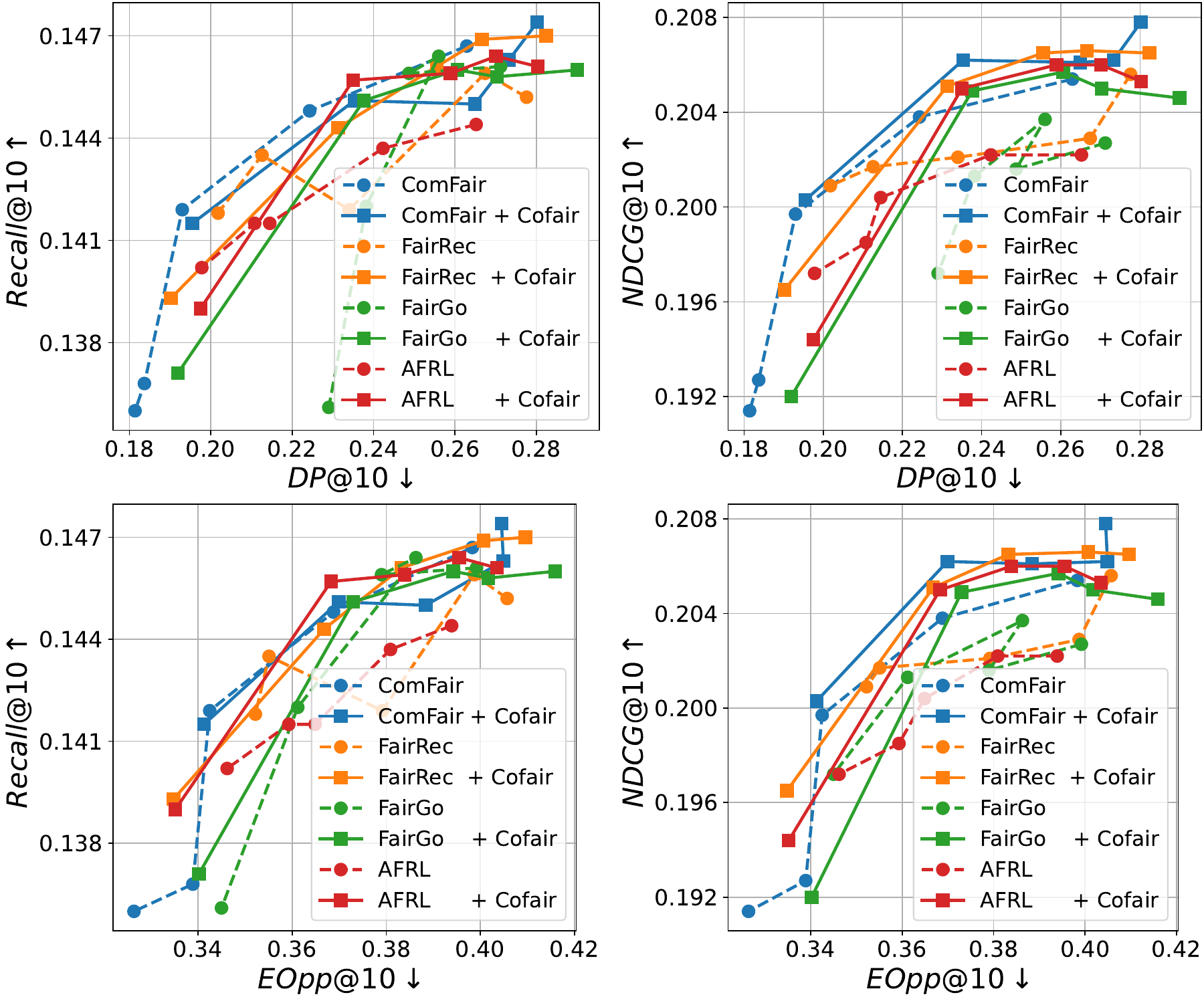}
    \vspace{-2mm}
    \caption{Comparison of four fairness methods without and with our \emph{controllable fairness} framework (indicated by “+ Cofair”), using BPR as backbone on MovieLens-1M.}
    \label{fig:framework_study}
\end{figure}

%

We validate the generality of Cofair by integrating it with four representative methods: ComFair, FairRec, FairGo, and AFRL. Figure~\ref{fig:framework_study} compares the original methods (which require retraining for each point) against their Cofair-augmented versions (``+ Cofair''). We highlight two key observations. Firstly, unlike the baselines that necessitate repeated retraining, the augmented methods achieve dynamic fairness control in a single run. By simply adjusting $t$ during inference, Cofair spans a broader range of fairness values, enabling fine-grained control based on real-time needs. Secondly, the augmented methods consistently maintain or surpass the Pareto efficiency of the original baselines. Notably, they often achieve superior fairness metrics (lower DP@10 or EOpp@10) while preserving competitive accuracy (NDCG). This confirms that Cofair serves as an effective, plug-and-play framework to inject flexibility into existing models without performance degradation.

\vspace{-2mm}

\subsection{Efficiency Analysis (RQ5)}
\label{sec:efficiency_analysis}


\begin{table}[htpb]
\begin{center}
\small
\caption{Average training time per epoch (in seconds) and the average number of best epochs for fairness methods.
}  
\begin{tabular}{@{}llrrrrc@{}}
\toprule
\multicolumn{2}{c}{\multirow{2}{*}{\textbf{Method}}} &\multicolumn{2}{c}{Movielens-1M} &\multicolumn{2}{c}{Lastfm-360K}  \\ 
\cmidrule(lr){3-4} \cmidrule(lr){5-6}  & & Time & Epoch & Time & Epoch  \\ \midrule

\multirow{4}{*}{\textbf{BPR}} & + ComFair & 3.10 & 78.57 & 5.15 & 132.00 \\
& + FairRec & 3.06 & 44.00 & 16.36 & 141.67 \\
& + FairGo & 9.66 & 114.38 & 8.68 & 175.33 \\
& + AFRL & 2.81 & 57.14 & 7.56 & 103.50 \\
& + Cofair & 4.79 & 7.44 & 12.24 & 26.04 \\

\midrule

\multirow{4}{*}{\textbf{LightGCN}} & + ComFair & 3.64 & 129.00 & 23.93 & 134.50 \\
& + FairRec & 10.42 & 122.50 & 27.93 & 185.56 \\
& + FairGo & 13.26 & 99.00 & 22.57 & 215.67 \\
& + AFRL & 3.53 & 40.00 & 11.33 & 55.00 \\
& + Cofair & 7.55 & 11.04 & 27.79 & 35.58 \\

\bottomrule
\end{tabular}
\label{tbl:efficiency-results}
\end{center}
\vspace{-5mm}
\end{table}

To evaluate the practical efficiency of our Cofair relative to existing fairness baselines, we perform all experiments on the same machine equipped with a single NVIDIA Tesla V100 GPU. Table~\ref{tbl:efficiency-results} presents the training time per epoch and the average number of epochs required to produce five different fairness levels by each method.

Notably, Cofair requires only about one-fifth the number of epochs used by competing approaches, although it incurs approximately twice the per-epoch time relative to the fastest baseline. Nevertheless, this reduced epoch requirement leads to substantial overall time savings, demonstrating that Cofair can efficiently generate multiple fairness configurations within a single training cycle. These findings highlight Cofair’s scalability to large recommender systems, where repeated retraining to explore multiple fairness-accuracy trade-offs would be computationally prohibitive.


\section{Related Work}
\label{sec:related}

\subsection{Fairness in Recommendation}

Extensive studies have highlighted that recommender systems can perpetuate and amplify societal biases, leading to unfair treatment of users based on sensitive attributes~\cite{yoo2024ensuring, xu2024fairsync, zhang2023fairlisa, chen2025leave}. 
Notably, recent work has comprehensively investigated such user-side fairness across outcome and process dimensions~\cite{chen2025investigating}.
Various fairness definitions have been proposed to address these concerns~\cite{wang2022survey, lifairness}. 
Individual fairness ensures similar predictions for similar individuals regardless of their sensitive attributes~\cite{biega2018equity}, while envy-free fairness~\cite{EnvyFree} requires that users should be free of envy on others' recommendations over their own. 
Counterfactual fairness~\cite{DBLP:conf/nips/KusnerLRS17, PCFR} requires consistent recommendation distributions across actual and counterfactual worlds where users' sensitive attributes are intervened.
Among these, group fairness has emerged as the most extensively studied paradigm due to its intuitive interpretation and direct focus on equal treatment across demographic groups~\cite{zemel2013learning, fairmi}. Group fairness primarily addresses equity among user groups with different sensitive attributes in terms of recommendation distribution or performance metrics~\cite{lifairness, wang2022survey}. Notably, demographic parity~\cite{kamishima2011fairness}, which promotes similar treatment across groups in both rating-based~\cite{kamishima2011fairness} and ranking-based~\cite{FairGo} scenarios.
Equal opportunity~\cite{DBLP:conf/nips/HardtPNS16} considers true user preferences in addition, with corresponding metrics for rating-based~\cite{FOCF} and ranking-based~\cite{fairmi} systems.
To achieve group fairness, researchers have developed various approaches, including regularization-based methods~\cite{FOCF, togashi2024scalable, shao2024average}, adversarial learning techniques~\cite{bose2019compositional, PCFR, yang2024distributional}, and re-ranking strategies~\cite{UGF, xu2023p}. 
For instance, FOCF~\cite{FOCF} directly optimizes fairness metrics through regularization, while ComFair~\cite{bose2019compositional} employs adversarial networks to eliminate sensitive information from user representations. 
UGF~\cite{UGF} addresses user unfairness through integer programming-based re-ranking. 
However, these approaches generally lack post-training flexibility, requiring complete model retraining to accommodate different fairness requirements.

\vspace{-2mm}

\subsection{Controllable Fairness}

Recent advances in controllable fairness have pursued two main directions. The first focuses on attribute-level control, allowing stakeholders to select which sensitive attributes to protect after training. Li et al.~\cite{PCFR} pioneered personalized fairness by enabling users to specify which sensitive attributes should be independent of their recommendations. AFRL~\cite{zhu2024adaptive} extended this concept through information alignment, treating fairness requirements as inputs for adaptive fair representation learning. However, these approaches offer limited control over the magnitude of fairness enforcement during inference.
The second direction emphasizes theoretical guarantees through constrained optimization. Several works~\cite{song2019learning, cui2023controllable, gupta2021controllable} achieve controllable fairness by allowing stakeholders to specify unfairness limits prior to training. For example,~\cite{song2019learning} derives tractable bounds connecting to specified fairness constraints for controllable optimization, while~\cite{cui2023controllable} enables both unfairness limit specification during training and demographic group selection post-training. While these approaches provide theoretical guarantees through constrained optimization, their controllability remains confined to the training phase, necessitating model retraining to achieve different fairness levels in practice.


\section{Conclusion}
\label{sec:conclusion}

In this paper, we addressed the critical challenge of post-training fairness inflexibility in recommender systems by introducing Cofair, a single-train framework enabling dynamic fairness adjustments after training. Our approach eliminates the need for repeated full retraining whenever stakeholders adjust fairness requirements. The framework's architecture combines a shared representation layer with fairness-conditioned adapters, effectively capturing both universal user characteristics and fairness-specific patterns across different fairness levels. The shared representation layer establishes a foundation for balancing accuracy and fairness across various settings, while the fairness-conditioned adapters fine-tune user representations according to specific fairness requirements. To ensure monotonic fairness enforcements, we implemented a user-level regularization loss that prevents fairness degradation for individual users as fairness levels increase. Our theoretical analysis demonstrates that Cofair's adversarial fairness objective provides an upper bound for the fairness criterion (e.g., demographic parity), while the user-level regularization ensures monotonic enhancement of fairness metrics. Through comprehensive experimental evaluation against state-of-the-art baselines, we validated that Cofair achieves controllable fairness across multiple levels while maintaining competitive fairness-accuracy trade-offs, all without the computational burden of model retraining.

Future research directions could explore the integration of multiple fairness definitions into a unified framework~\cite{wu2022multi, wang2024intersectional, wu2021tfrom}. While our current approach can be adapted to different fairness notions, developing a structured framework that explicitly leverages the relationships among various fairness definitions remains challenging. 


\begin{acks}
This work is supported by Hong Kong Baptist University Key Research Partnership Scheme (KRPS/23-24/02) and NSFC/RGC Joint Research Scheme (N\_HKBU214/24).

\end{acks}

\balance
\clearpage
\bibliographystyle{ACM-Reference-Format}
\bibliography{8Reference}

@article{ZCH25,
author = {Zhao, Yuhan and Chen, Rui and Han, Qilong and Song, Hongtao and Chen, Li},
title = {Unlocking the unlabeled data: Enhancing recommendations with neutral Samples and uncertainty},
year = {2025},
journal = {ACM Transactions on Recommender Systems}
}

@inproceedings{ZCC24,
author = {Zhao, Yuhan and Chen, Rui and Chen, Li and Zhang, Shuang and Han, Qilong and Song, Hongtao},
title = {From pairwise to ranking: Climbing the ladder to ideal collaborative filtering with pseudo-ranking},
year = {2025},
booktitle = {Proceedings of the 39th AAAI Conference on Artificial Intelligence (AAAI'25)},
pages = {13392-13400}
}

@inproceedings{ZCL23,
author = {Zhao, Yuhan and Chen, Rui and Lai, Riwei and Han, Qilong and Song, Hongtao and Chen, Li},
title = {Augmented Negative Sampling for Collaborative Filtering},
booktitle = {Proceedings of the 17th ACM conference on recommender systems (RecSys'23)},
pages = {256–266},
year = {2023}
}

@inproceedings{ZCH24,
author = {Zhao, Yuhan and Chen, Rui and Han, Qilong and Song, Hongtao and Chen, Li},
title = {Unlocking the hidden treasures: enhancing recommendations with unlabeled data},
year = {2024},
booktitle = {Proceedings of the 18th ACM Conference on Recommender Systems (RecSys'24)},
pages = {247–256},
}

@inproceedings{fairmi,
  title={Fair representation learning for recommendation: A mutual information perspective},
  author={Zhao, Chen and Wu, Le and Shao, Pengyang and Zhang, Kun and Hong, Richang and Wang, Meng},
  booktitle={Proceedings of the 37th AAAI Conference on Artificial Intelligence (AAAI'23)},
  year={2023},
  pages= {4911--4919},
}

@article{gunawardana2009survey,
  title={A survey of accuracy evaluation metrics of recommendation tasks},
  author={Gunawardana, Asela and Shani, Guy},
  journal={Journal of Machine Learning Research},
  volume={10},
  number={12},
  year={2009}
}

@inproceedings{jarvelin2017ir,
  title={IR evaluation methods for retrieving highly relevant documents},
  author={J{\"a}rvelin, Kalervo and Kek{\"a}l{\"a}inen, Jaana},
  booktitle={ACM SIGIR Forum},
  volume={51},
  number={2},
  pages={243--250},
  year={2017},
}

@article{harper2015movielens,
  title={The movielens datasets: History and context},
  author={Harper, F Maxwell and Konstan, Joseph A},
  journal={Acm Transactions on Interactive Intelligent Systems},
  volume={5},
  number={4},
  pages={1--19},
  year={2015},
}

@book{celma2009music,
  title={Music recommendation and discovery in the long tail},
  author={Celma Herrada, {\`O}scar and others},
  year={2009},
  publisher={Universitat Pompeu Fabra}
}

@inproceedings{islam2021debiasing,
  title={Debiasing career recommendations with neural fair collaborative filtering},
  author={Islam, Rashidul and Keya, Kamrun Naher and Zeng, Ziqian and Pan, Shimei and Foulds, James},
  booktitle={Proceedings of the Web Conference 2021 (WWW'21)},
  pages={3779--3790},
  year={2021}
}

@article{wang2022survey,
	title        = {A survey on the fairness of recommender systems},
	author       = {Wang, Yifan and Ma, Weizhi and Zhang, Min and Liu, Yiqun and Ma, Shaoping},
    journal={ACM Transactions on Information Systems},
    year={2023},
      volume       = {41},
      number       = {3},
  pages        = {52:1--52:43},
}

@article{lifairness,
	title        = {Fairness in recommendation: Foundations, methods and applications},
	author       = {Li, Yunqi and Chen, Hanxiong and Xu, Shuyuan and Ge, Yingqiang and Tan, Juntao and Liu, Shuchang and Zhang, Yongfeng},
	year         = 2023,
	journal      = {ACM Transactions on Intelligent Systems and Technology},
  volume       = {14},
  number       = {5},
  pages        = {95:1--95:48},

}

@article{deldjoo2024fairness,
  title={Fairness in recommender systems: research landscape and future directions},
  author={Deldjoo, Yashar and Jannach, Dietmar and Bellogin, Alejandro and Difonzo, Alessandro and Zanzonelli, Dario},
  journal={User Modeling and User-Adapted Interaction (UMUAI'24)},
  volume={34},
  number={1},
  pages={59--108},
  year={2024},
}

@inproceedings{BPR,
	title        = {{BPR:} Bayesian personalized ranking from implicit feedback},
	author       = {Steffen Rendle and Christoph Freudenthaler and Zeno Gantner and Lars Schmidt{-}Thieme},
	year         = 2009,
	booktitle    = {Proceedings of the 25th Conference on Uncertainty in Artificial Intelligence (UAI'09)},
  pages        = {452--461},
}

@inproceedings{he2020lightgcn,
	title        = {LightGCN: Simplifying and powering graph convolution network for recommendation},
	author       = {He, Xiangnan and Deng, Kuan and Wang, Xiang and Li, Yan and Zhang, Yongdong and Wang, Meng},
	year         = 2020,
	booktitle    = {Proceedings of the 43rd International {ACM} {SIGIR} Conference on Research and Development in Information Retrieval (SIGIR'20)},
  pages        = {639--648},
}

@inproceedings{bose2019compositional,
	title        = {Compositional fairness constraints for graph embeddings},
	author       = {Bose, Avishek and Hamilton, William},
	year         = 2019,
	booktitle    = {Proceedings of the 36th International Conference on Machine Learning (ICML'19)},
      pages        = {715--724},
}

@inproceedings{FairGo,
	title        = {Learning fair Representations for recommendation: {A} graph-based perspective},
	author       = {Le Wu and Lei Chen and Pengyang Shao and Richang Hong and Xiting Wang and Meng Wang},
	year         = 2021,
	booktitle    = {Proceedings of the Web Conference (WWW'21)},
    pages        =  {2198--2208},
}

@inproceedings{FairRec,
	title        = {Fairness-aware news recommendation with decomposed adversarial learning},
	author       = {Chuhan Wu and Fangzhao Wu and Xiting Wang and Yongfeng Huang and Xing Xie},
	year         = 2021,
	booktitle    = {Proceedings of the 35th AAAI Conference on Artificial Intelligence (AAAI'21)},
  pages        = {4462--4469},
}

@inproceedings{zhu2024adaptive,
author = {Zhu, Xinyu and Zhang, Lilin and Yang, Ning},
title = {Adaptive fair Representation learning for personalized fairness in recommendations via information alignment},
year = {2024},
booktitle = {Proceedings of the 47th International {ACM} {SIGIR} Conference on Research and Development in Information Retrieval (SIGIR'24)},
  pages        = {427--436},
}

@inproceedings{DBLP:journals/corr/KingmaB14,
	title        = {Adam: {A} method for stochastic optimization},
	author       = {Diederik P. Kingma and Jimmy Ba},
	year         = 2015,
	booktitle    = {Proceedings of the 3rd International Conference on Learning Representations (ICLR'15)},
}

@inproceedings{glorot2010understanding,
	title        = {Understanding the difficulty of training deep feedforward neural networks},
	author       = {Glorot, Xavier and Bengio, Yoshua},
	year         = 2010,
	booktitle    = {Proceedings of the 13th International Conference on Artificial Intelligence and Statistics (AISTATS'10)},
    pages       = {249-256},
}

@inproceedings{wei2021contrastive,
  title={Contrastive learning for cold-start recommendation},
  author={Wei, Yinwei and Wang, Xiang and Li, Qi and Nie, Liqiang and Li, Yan and Li, Xuanping and Chua, Tat-Seng},
  booktitle={Proceedings of the 29th ACM International Conference on Multimedia (MM'21)},
  pages={5382--5390},
  year={2021}
}

@inproceedings{maas2013rectifier,
  title={Rectifier nonlinearities improve neural network acoustic models},
  author={Maas, Andrew L and Hannun, Awni Y and Ng, Andrew Y and others},
  booktitle={Proceedings of the 30th International Conference on Machine Learning (ICML'13)},
  pages={3},
  year={2013},
}

@inproceedings{zemel2013learning,
  title={Learning fair representations},
  author={Zemel, Rich and Wu, Yu and Swersky, Kevin and Pitassi, Toni and Dwork, Cynthia},
  booktitle={Proceedings of the 30th International Conference on Machine Learning (ICML'13)},
  pages={325--333},
  year={2013},
}

@inproceedings{dwork2012fairness,
  title={Fairness through awareness},
  author={Dwork, Cynthia and Hardt, Moritz and Pitassi, Toniann and Reingold, Omer and Zemel, Richard},
  booktitle={Proceedings of the 3rd Innovations in Theoretical Computer Science Conference (ITCS'12)},
  pages={214--226},
  year={2012}
}

@inproceedings{FOCF,
	title        = {Beyond parity: Fairness objectives for collaborative filtering},
	author       = {Sirui Yao and Bert Huang},
	year         = 2017,
    booktitle    = {Proceedings of the 30th International Conference on Neural Information Processing Systems (NeurIPS'17)},
    pages        = {2921--2930},
}

@inproceedings{covington2016deep,
  title={Deep neural networks for youtube recommendations},
  author={Covington, Paul and Adams, Jay and Sargin, Emre},
  booktitle={Proceedings of the 10th ACM Conference on Recommender Systems (RecSys'16)},
  pages={191--198},
  year={2016}
}

@article{linden2003amazon,
  title={Amazon. com recommendations: Item-to-item collaborative filtering},
  author={Linden, Greg and Smith, Brent and York, Jeremy},
  journal={IEEE Internet Computing},
  volume={7},
  number={1},
  pages={76--80},
  year={2003},
}

@inproceedings{wang2018billion,
  title={Billion-scale commodity embedding for e-commerce recommendation in alibaba},
  author={Wang, Jizhe and Huang, Pipei and Zhao, Huan and Zhang, Zhibo and Zhao, Binqiang and Lee, Dik Lun},
  booktitle={Proceedings of the 24th ACM SIGKDD International Conference on Knowledge Discovery \& Data Mining (KDD'18)},
  pages={839--848},
  year={2018}
}

@inproceedings{guy2010social,
  title={Social media recommendation based on people and tags},
  author={Guy, Ido and Zwerdling, Naama and Ronen, Inbal and Carmel, David and Uziel, Erel},
  booktitle={Proceedings of the 33rd international ACM SIGIR conference on Research and Development in Information Retrieval (SIGIR'10)},
  pages={194--201},
  year={2010}
}

@inproceedings{DBLP:conf/sigir/WuXZZ0ZL022,
	title        = {Selective fairness in recommendation via prompts},
	author       = {Yiqing Wu and Ruobing Xie and Yongchun Zhu and Fuzhen Zhuang and Xiang Ao and Xu Zhang and Leyu Lin and Qing He},
	year         = 2022,
	booktitle    = {Proceedings of the 45th International ACM SIGIR Conference on Research and Development in Information
Retrieval (SIGIR'22)},
  pages        = {2657--2662},
}

@inproceedings{DBLP:conf/icml/MadrasCPZ18,
  author       = {David Madras and
                  Elliot Creager and
                  Toniann Pitassi and
                  Richard S. Zemel},
  title        = {Learning adversarially fair and transferable representations},
  booktitle    = {Proceedings of the 35th International Conference on Machine Learning (ICML'18)},
  pages        = {3381--3390},
  year         = {2018},
}

@inproceedings{PCFR,
	title        = {Towards personalized fairness based on causal notion},
	author       = {Yunqi Li and Hanxiong Chen and Shuyuan Xu and Yingqiang Ge and Yongfeng Zhang},
	year         = 2021,
	booktitle    = {Proceedings of the 44th International ACM SIGIR Conference on Research and Development in Information
Retrieval (SIGIR'21)},
    pages = {1054--1063},
}

@inproceedings{song2019learning,
  title={Learning controllable fair representations},
  author={Song, Jiaming and Kalluri, Pratyusha and Grover, Aditya and Zhao, Shengjia and Ermon, Stefano},
  booktitle={The 22nd International Conference on Artificial Intelligence and Statistics (AISTATS'19)},
  pages={2164--2173},
  year={2019},
}

@inproceedings{cui2023controllable,
  title={Controllable universal fair representation learning},
  author={Cui, Yue and Chen, Ma and Zheng, Kai and Chen, Lei and Zhou, Xiaofang},
  booktitle={Proceedings of the ACM Web Conference 2023 (WWW'23)},
  pages={949--959},
  year={2023}
}

@inproceedings{DBLP:conf/nips/HardtPNS16,
	title        = {Equality of opportunity in supervised learning},
	author       = {Moritz Hardt and Eric Price and Nati Srebro},
	year         = 2016,
    booktitle    = {Proceedings of the 29th International Conference on Neural Information Processing Systems (NeurIPS'16)},
    pages       = {3315--3323},
}

@inproceedings{DBLP:conf/nips/KusnerLRS17,
	title        = {Counterfactual fairness},
	author       = {Matt J. Kusner and Joshua R. Loftus and Chris Russell and Ricardo Silva},
	year         = {2017},
    booktitle    = {Proceedings of the 30th International Conference on Neural Information Processing Systems (NeurIPS'17)},
    pages        = {4066--4076},
}

@inproceedings{EnvyFree,
  title={Fair allocation of indivisible goods: Improvements and generalizations},
  author={Ghodsi, Mohammad and HajiAghayi, MohammadTaghi and Seddighin, Masoud and Seddighin, Saeed and Yami, Hadi},
  booktitle={Proceedings of the ACM Conference on Economics
and Computation (EC'18)},
  pages={539--556},
  year={2018}
}

@inproceedings{biega2018equity,
  title={Equity of attention: Amortizing individual fairness in rankings},
  author={Biega, Asia J and Gummadi, Krishna P and Weikum, Gerhard},
  booktitle={Proceedings of the 41st International ACM SIGIR Conference on Research and Development in Information Retrieval (SIGIR'18)},
  pages={405--414},
  year={2018}
}

@inproceedings{xu2024fairsync,
	title        = {FairSync: Ensuring amortized group exposure in distributed recommendation retrieval},
	author       = {Xu, Chen and Xu, Jun and Ding, Yiming and Zhang, Xiao and Qi, Qi},
	year         = 2024,
	booktitle    = {Proceedings of the ACM on Web Conference (WWW'24)},
	pages        = {1092--1102}
}

@inproceedings{yoo2024ensuring,
	title        = {Ensuring user-side fairness in dynamic recommender systems},
	author       = {Yoo, Hyunsik and Zeng, Zhichen and Kang, Jian and Qiu, Ruizhong and Zhou, David and Liu, Zhining and Wang, Fei and Xu, Charlie and Chan, Eunice and Tong, Hanghang},
	year         = 2024,
	booktitle    = {Proceedings of the ACM on Web Conference 2024 (WWW'24)},
	pages        = {3667--3678}
}

@inproceedings{zhang2023fairlisa,
	title        = {Fairlisa: Fair user modeling with limited sensitive attributes information},
	author       = {Zhang, Zheng and Liu, Qi and Jiang, Hao and Wang, Fei and Zhuang, Yan and Wu, Le and Gao, Weibo and Chen, Enhong},
	year         = 2023,
	booktitle    = {Proceedings of the 37th Conference on Neural Information Processing Systems (NeurIPS'23)}
}

@inproceedings{kamishima2011fairness,
  title={Fairness-aware learning through regularization approach},
  author={Kamishima, Toshihiro and Akaho, Shotaro and Sakuma, Jun},
  booktitle={2011 IEEE 11th international conference on data mining workshops},
  pages={643--650},
  year={2011},
}

@inproceedings{UGF,
	title        = {User-oriented fairness in recommendation},
	author       = {Yunqi Li and Hanxiong Chen and Zuohui Fu and Yingqiang Ge and Yongfeng Zhang},
	year         = 2021,
	booktitle    = {Proceedings of the Web Conference (WWW'21)},
    pages = {624--632},
}

@inproceedings{gupta2021controllable,
  title={Controllable guarantees for fair outcomes via contrastive information estimation},
  author={Gupta, Umang and Ferber, Aaron M and Dilkina, Bistra and Ver Steeg, Greg},
  booktitle={Proceedings of the AAAI Conference on Artificial Intelligence (AAAI'21)},
  volume={35},
  number={9},
  pages={7610--7619},
  year={2021}
}

@inproceedings{togashi2024scalable,
  title={Scalable and provably fair exposure control for large-scale recommender systems},
  author={Togashi, Riku and Abe, Kenshi and Saito, Yuta},
  booktitle={Proceedings of the ACM on Web Conference (WWW'24)},
  pages={3307--3318},
  year={2024}
}

@article{shao2024average,
	title        = {Average user-side counterfactual fairness for collaborative filtering},
	author       = {Shao, Pengyang and Wu, Le and Zhang, Kun and Lian, Defu and Hong, Richang and Li, Yong and Wang, Meng},
	year         = 2024,
	journal      = {ACM Transactions on Information Systems},
	volume       = 42,
	number       = 5,
	pages        = {1--26}
}

@article{yang2024distributional,
	title        = {Distributional fairness-aware recommendation},
	author       = {Yang, Hao and Wu, Xian and Qiu, Zhaopeng and Zheng, Yefeng and Chen, Xu},
	year         = 2024,
	journal      = {ACM Transactions on Information Systems},
	volume       = 42,
	number       = 5,
	pages        = {1--28}
}

@inproceedings{xu2023p,
	title        = {P-MMF: Provider max-min fairness re-ranking in recommender system},
	author       = {Xu, Chen and Chen, Sirui and Xu, Jun and Shen, Weiran and Zhang, Xiao and Wang, Gang and Dong, Zhenhua},
	year         = 2023,
	booktitle    = {Proceedings of the ACM Web Conference 2023},
	pages        = {3701--3711}
}

@article{wu2022multi,
  title={A multi-objective optimization framework for multi-stakeholder fairness-aware recommendation},
  author={Wu, Haolun and Ma, Chen and Mitra, Bhaskar and Diaz, Fernando and Liu, Xue},
  journal={ACM Transactions on Information Systems},
  volume={41},
  number={2},
  pages={1--29},
  year={2022},
}

@inproceedings{wang2024intersectional,
  title={Intersectional two-sided fairness in recommendation},
  author={Wang, Yifan and Sun, Peijie and Ma, Weizhi and Zhang, Min and Zhang, Yuan and Jiang, Peng and Ma, Shaoping},
  booktitle={Proceedings of the ACM on Web Conference (WWW'24)},
  pages={3609--3620},
  year={2024}
}

@inproceedings{wu2021tfrom,
  title={Tfrom: A two-sided fairness-aware recommendation model for both customers and providers},
  author={Wu, Yao and Cao, Jian and Xu, Guandong and Tan, Yudong},
  booktitle={Proceedings of the 44th international ACM SIGIR conference on research and development in information retrieval (SIGIR'21)},
  pages={1013--1022},
  year={2021}
}

@article{caruana1997multitask,
  title={Multitask learning},
  author={Caruana, Rich},
  journal={Machine learning},
  volume={28},
  pages={41--75},
  year={1997},
}

@inproceedings{wang2021understanding,
  title={Understanding and improving fairness-accuracy trade-offs in multi-task learning},
  author={Wang, Yuyan and Wang, Xuezhi and Beutel, Alex and Prost, Flavien and Chen, Jilin and Chi, Ed H},
  booktitle={Proceedings of the 27th ACM SIGKDD Conference on Knowledge Discovery \& Data Mining (KDD'21)},
  pages={1748--1757},
  year={2021}
}

@article{kusupati2022matryoshka,
  title={Matryoshka representation learning},
  author={Kusupati, Aditya and Bhatt, Gantavya and Rege, Aniket and Wallingford, Matthew and Sinha, Aditya and Ramanujan, Vivek and Howard-Snyder, William and Chen, Kaifeng and Kakade, Sham and Jain, Prateek and others},
  journal={Proceedings of the 36th Conference on Neural Information Processing Systems (NeurIPS'22)},
  volume={35},
  pages={30233--30249},
  year={2022}
}

@inproceedings{goodfellow2014generative,
	title        = {Generative adversarial nets},
	author       = {Goodfellow, Ian and Pouget-Abadie, Jean and Mirza, Mehdi and Xu, Bing and Warde-Farley, David and Ozair, Sherjil and Courville, Aaron and Bengio, Yoshua},
	year         = 2014,
	booktitle      = {Proceedings of the 27th International Conference on Neural Information Processing Systems (NeurIPS'14)},
  pages        = {2672--2680},
}

@article{chen2025causality,
  title     = {Causality-inspired fair representation learning for multimodal recommendation},
  author    = {Chen, Weixin and Chen, Li and Ni, Yongxin and Zhao, Yuhan},
  year      = 2025,
  journal   = {ACM Transactions on Information Systems},
  volume    = {43},
  number    = {6},
  articleno = {153},
  numpages  = {29}
}

@inproceedings{chen2025leave,
  title     = {Leave no one behind: Fairness-aware cross-domain recommender systems for non-overlapping users},
  author    = {Chen, Weixin and Zhao, Yuhan and Chen, Li and Pan, Weike},
  year      = 2025,
  booktitle = {Proceedings of the 19th ACM Conference on Recommender Systems (RecSys'25)},
  pages     = {226--236}
}

@article{chen2025investigating,
  title     = {Investigating User-side fairness in outcome and process for multi-type sensitive attributes in recommendations},
  author    = {Chen, Weixin and Chen, Li and Zhao, Yuhan},
  year      = 2025,
  journal   = {ACM Transactions on Recommender Systems},
  volume    = {4},
  number    = {2},
  articleno = {25},
  numpages  = {29}
}

\clearpage

\end{document}